 \documentclass[review,1p,times]{elsarticle}

\usepackage{amssymb}
\usepackage{mathptmx}      

\DeclareMathAlphabet{\mathcal}{OMS}{cmsy}{m}{n} 
\usepackage{latexsym}
\usepackage{graphicx}
\usepackage{pbox}
\usepackage{amsfonts,amssymb}
\usepackage{amsmath}
\usepackage{amsthm}
\usepackage{changepage}
\usepackage{tikz}
\usepackage{soul}
\usepackage{multirow}
\usepackage{booktabs}
\usepackage{threeparttable}
\usepackage{floatrow}
\usepackage{algorithm}
\usepackage{algorithmic}
\usepackage{tabularx}
\usepackage[titletoc,title]{appendix}
\usepackage{hyperref}
\usepackage{pbox}
\usepackage{subfig}
\usepackage{natbib}
\hypersetup{colorlinks,linkcolor={blue},citecolor={blue},urlcolor={blue}}
\usepackage{url}
\usepackage{bm}
\DeclareMathOperator*{\argmin}{arg\,min}
\newtheorem{proposition}{Proposition}[section]

\usepackage[ampersand]{easylist}

\begin{document}

\begin{frontmatter}



\title{Cross-Domain Structure Preserving Projection for Heterogeneous Domain Adaptation}


\author{Qian Wang,  Toby P. Breckon}

\address{Department of Computer Science, Durham University, UK. \\qian.wang173@hotmail.com, toby.breckon@durham.ac.uk }

\begin{abstract}
Heterogeneous Domain Adaptation (HDA) addresses the transfer learning problems where data from the source and target domains are of different modalities (e.g., texts and images) or feature dimensions (e.g., features extracted with different methods). It is useful for multi-modal data analysis. Traditional domain adaptation algorithms assume that the representations of source and target samples reside in the same feature space, hence are likely to fail in solving the heterogeneous domain adaptation problem. Contemporary state-of-the-art HDA approaches are usually composed of complex optimization objectives for favourable performance and are therefore computationally expensive and less generalizable. 
To address these issues, we propose a novel Cross-Domain Structure Preserving Projection (CDSPP) algorithm for HDA. As an extension of the classic LPP to heterogeneous domains, CDSPP aims to learn domain-specific projections to map sample features from source and target domains into a common subspace such that the class consistency is preserved and data distributions are sufficiently aligned. CDSPP is simple and has deterministic solutions by solving a generalized eigenvalue problem. It is naturally suitable for supervised HDA but has also been extended for semi-supervised HDA where the unlabelled target domain samples are available. Extensive experiments have been conducted on commonly used benchmark datasets (i.e. Office-Caltech, Multilingual Reuters Collection, NUS-WIDE-ImageNet) for HDA as well as the Office-Home dataset firstly introduced for HDA by ourselves due to its significantly larger number of classes than the existing ones (65 vs 10, 6 and 8). 
The experimental results of both supervised and semi-supervised HDA demonstrate the superior performance of our proposed method against contemporary state-of-the-art methods.
\end{abstract}


\end{frontmatter}


\section{Introduction} \label{sec:intro}
Supervised learning can achieve good performance given considerable amounts of labelled data for training. One essential factor accounting for the recent successes in deep learning and image classification is the ImageNet database which contains more than 14 million hand-annotated images \cite{deng2009imagenet}. However, there exist many tasks in real-world applications where sufficient labelled data are not available, hence the performance of traditional supervised learning approaches can degrade significantly. One promising technique alleviating this problem is transfer learning which aims to transfer knowledge learned from the source domain to the target domain in which labelled data are sparse and expensive to collect \citep{weiss2016survey}. In many scenarios, domain adaptation is required since the data distributions in the source and target domains can be different and the models trained with source domain data are not directly applicable to the target domain \citep{patel2015visual}. 

Since domain adaptation is a promising solution to the training data sparsity issue in many real-world applications, it has been studied in a variety of research tasks including image classification \citep{wang2019unifying}, semantic segmentation \citep{zhao2019multi}, depth estimation \citep{atapour2018real}, speech emotion recognition \citep{zhou2019transferable}, text classification \citep{zhou2019multi} and many others.

Domain adaptation approaches aim to model the domain shift between source and target domains and reduce the discrepancy by aligning the data distributions \citep{wang2019unifying,wang2020unsupervised}. In the scope of classification problems, this is usually boiled down to aligning the marginal and class conditional distributions across domains \citep{wang2018visual,chen2018joint}. However, most existing works are based on the assumption of homogeneity, i.e., the source and target data are represented in the same feature space with unaligned distributions \citep{zhao2019multi,wang2019unifying,zhang2019domain,wang2020unsupervised}. These approaches may not be applicable in situations where the source and target domains are \textit{heterogeneous} in the forms of data modalities (e.g., texts vs images) or representations (e.g., features extracted with different methods).

Attempts have been made to extend the success of domain adaptation approaches to the HDA problems, however, it is non-trivial for the common subspace learning methods due to the heterogeneous feature spaces across the source and target domains. One common solution to such extension is to learn two domain-specific projections instead of one unified projection for the source and target domains in HDA problems \cite{wang2011heterogeneous,li2018heterogeneous}. Nevertheless, there are at least two limitations in these existing methods. One is most of them use Maximum Mean Discrepancy (MMD) as the objective to learn the projection matrices. MMD based objectives have been outperformed by more recent ones based on locality preserving projection \cite{wang2020unsupervised,li2019locality} in homogeneous domain adaptation. In HDA problems, locality preserving objectives have not been well explored despite some attempts in \cite{wang2011heterogeneous,li2018heterogeneous}. In this paper, we present a succinct yet effective algorithm by extending the locality preserving objectives for heterogeneous domain adaptation. The other limitation of existing HDA approaches is the way how they exploit the unlabelled target-domain data are sub-optimal. In our work, we propose a novel selective pseudo-labelling strategy to take advantage of the unlabelled target-domain data. The selection is based on the classification confidence and applies to a variety of classification models (e.g., Nearest Neighbour, SVM and Neural Networks).

Specifically, we address the heterogeneous domain adaptation problem where the source and target data are represented in heterogeneous feature spaces. Following the same spirits of previous domain adaptation approaches \citep{wang2018visual,wang2019unifying,wang2020unsupervised}, we try to learn a common latent subspace where both source and target data can be projected and well aligned in the learnt subspace. Specifically, we learn domain-specific projections using a novel Cross-Domain Structure Preserving Projection (CDSPP) algorithm which is an extension of the classic Locality Preserving Projection (LPP) algorithm \citep{he2004locality}. CDSPP can facilitate class consistency preserving to learn domain-specific projections which can be used to map heterogeneous data representations into a common subspace for recognition. CDSPP is simple yet effective in solving the heterogeneous domain adaptation problem as empirically validated by our experimental results on several benchmark datasets. To take advantage of the unlabelled target-domain data in the semi-supervised HDA setting, a selective pseudo-labelling strategy is employed to progressively optimise the projections and target data label predictions. The contributions of this work can be summarised as follows:
\begin{itemize}
    \item[-] A novel Cross-Domain Structure Preserving Projection algorithm is proposed for heterogeneous domain adaptation and the algorithm has a concise solution by solving a generalized eigenvalue problem;
    \item[-] The proposed CDSPP algorithm is naturally for supervised HDA and we extend it to solve the semi-supervised HDA problems by employing an iterative pseudo-labelling approach;
    \item[-] We validate the effectiveness of the proposed method on several benchmark datasets including the newly introduced Office-Home which contains much more classes than the previously used ones; the experimental results provide evidence our algorithm outperforms prior art.
\end{itemize}
\section{Related Work} \label{sec:related}
Most exiting research in domain adaptation for classification is based on the assumption of homogeneity \cite{wang2020unsupervised,li2019locality,li2020maximum}. The approaches are dedicated to either learning a domain-invariant feature extraction model (e.g., deep CNN \citep{chen2019progressive,zhang2019domain}) or learning a unified feature projection matrix \citep{wang2018visual,wang2019unifying,wang2020unsupervised} for all domains whilst neither of them applies to HDA. In this section, we briefly review related works on heterogeneous domain adaptation.
The existing approaches to HDA can be roughly categorized into \textit{cross-domain mapping} and \textit{common subspace learning}.

\subsection{Cross-Domain Mapping}
Cross-domain mapping approaches learn a projection from the source to the target domain. The projection can be learned for either \textit{feature transformation} \citep{hubert2016learning,shen2018unsupervised} or \textit{model parameter transformation} (e.g., SVM weights \citep{zhou2019multi,mozafari2016svm}). Feature transformation approaches learn a projection to map the source data into the target data by aligning the data distribution \citep{hubert2016learning} or the second-order moment \citep{shen2018unsupervised}. As a result, the transformed source data can help to learn a classifier for the target domain. To avoid mapping a lower-dimensional feature to a higher-dimensional space, PCA is usually employed to learn subspaces for both domains respectively \citep{hubert2016learning} as a pre-processing which can suffer from information loss.

Model parameter transformation approaches focus mainly on SVM classifier weights. For a multi-class classification problem, one-vs-all classifiers are learned for source and target domains using the respective labelled samples. Subsequently, the cross-domain mapping is learned from the paired class-level weight vectors \citep{zhou2019multi,mozafari2016svm}. Since the number of classes is far less than the number of samples, these approaches are more computationally efficient but rely too much on the learned classifiers and overlooked abundant information underlying the data distribution.

\subsection{Common Subspace Learning}
Common subspace learning is a more popular strategy for HDA. It learns domain-specific projections which map source and target domain data into a common subspace. 
To this end, different approaches have been proposed with varying algorithms, e.g., Manifold Alignment \citep{wang2011heterogeneous,li2018transfer,fang2018discriminative,wu2021heterogeneous}, Canonical Correlation Analysis \citep{yan2017learning}, Coding Space Learning \citep{li2017locality,li2018heterogeneous,deng2019multiclass}, Deep Matrix Completion \citep{li2019heterogeneous} and Deep Neural Networks \citep{zhou2019deep,yao2019heterogeneous}.
Despite the diversity of implementation, the main objective of common subspace learning based HDA is similar, i.e., the alignment of the source and target domains. 

To align the distributions, \citep{hubert2016learning,li2017locality,li2018heterogeneous,li2018transfer,li2019heterogeneous}  chose to minimize the Maximum Mean Discrepancy (MMD) in their objectives which, however, can only align the means of domains (for marginal distributions) and the means of classes (for conditional distributions). As a result, the subspace learned via minimizing the MMD is not sufficiently discriminative. One alternative to MMD is the manifold learning using graph Laplacian \citep{wang2011heterogeneous,li2018heterogeneous,li2018transfer}. 

Li et al. \citep{li2013learning} proposed a Heterogeneous Feature Augmentation (HFA) method and its semi-supervised version SHFA by learning domain-specific projections and a classifier (i.e. SVM) simultaneously. However, the computational complexity is $\mathcal{O}(n^3)$, where $n$ is the number of labelled samples and makes it extremely slow when $n$ is large.
Li et al. \citep{li2017locality} learned new feature representations for source and target data by encoding them with a shared codebook which requires the original features have the same dimensions for source and target domains. PCA was employed for this purpose as a pre-processing but can suffer from information loss. Lately, the authors incorporated the learning of two domain-specific projections (in place of PCA) into the coding framework \citep{li2018heterogeneous}. This work is similar to ours in the sense of local consistency using the graph regularization, however, it fails to align cross-domain class consistency due to the use of $k$ nearest neighbours to construct the similarity graph. In our work, the similarity graph is constructed based on class consistency, hence promoting the cross-domain conditional distribution alignment.

Transfer Independently Together (TIT) was proposed in \citep{li2018transfer}. It also learns domain-specific projections to align data distributions in the learned common subspace. The algorithm was based on a collection of tricks including kernel space, MMD, sample reweighting and landmark selection.  In contrast, our solution is concise with one simple objective of cross-domain structure preserving. Recently, Huang et al. \citep{huang2020heterogeneous} proposed a novel algorithm, named heterogeneous discriminative features learning and label propagation (HDL). This algorithm is similar to ours in that both tend to preserve structure information in the learned common subspace. However, different objectives have been formulated. Our algorithm explicitly promotes the intra-class similarity for both within-domain and cross-domain samples, whilst HDL fails to consider the intra-class similarity for samples from the same domain in their formulation. In addition, different strategies of unlabelled target sample exploration were employed in two algorithms.

In summary, although manifold learning has been well studied in HDA, the existing formulations for domain-specific projection learning are either inefficient or ineffective. Our approach solves this issue and addresses the HDA problem with a novel CDSPP algorithm.

\section{Method} \label{sec:method}

To facilitate our presentation, we firstly describe the heterogeneous domain adaptation problem and notations used throughout this paper. 
Given a labelled dataset $\mathcal{D}^s = \{(\bm{x}^s_i,y^s_i)\}, i = 1,2,...,n_s$ from the source domain $\mathcal{S}$, and a labelled dataset $\mathcal{D}^t = \{\bm{x}^t_i,y^t_i\}, i = 1,2,...,n_t$ from the target domain, $\bm{x}^s_i \in \mathbb{R}^{d_s}$ and $\bm{x}^{t}_i \in \mathbb{R}^{d_t}$  represent the feature vectors of $i$-th labelled samples in the source and target domains respectively; $d_s$ and $d_t$ are the dimensionalities of the source and target features; $y^s_i \in \mathcal{Y}$ and $y^t_i\in \mathcal{Y}$ denote the corresponding sample labels; $n_s$ and $n_t$ are the number of source and labelled target samples respectively. Let $\bm{X}^s \in \mathbb{R}^{d_s\times n_s}$ and $\bm{X}^t \in \mathbb{R}^{d_t\times n_t}$ be the feature matrices of labelled source and target data collectively, supervised HDA aims to learn a model from labelled source and target data, which can be used to classify samples from an unlabelled dataset $\mathcal{D}^u = \{\bm{x}^u_i\}, i = 1,2,...,n_u$ from the target domain, whose feature vectors can be collectively denoted as $\bm{X^u} \in \mathbb{R}^{d_t\times n_u}$.

The number of labelled target samples $n_t$ is usually very small, hence it is difficult to capture the data distribution in the target domain. Semi-supervised HDA takes advantage of the unlabelled target samples $\bm{X^u}$ during model training and can usually achieve better performance. 

In this section, we describe the CDSPP algorithm which is naturally for supervised heterogeneous domain adaptation but can be used to address the semi-supervised heterogeneous domain adaptation problem by incorporating it into an iterative learning framework \citep{wang2019unifying,wang2020unsupervised} as shown in Figure \ref{fig:framework}.

\begin{figure}
    \centering
    {\includegraphics[width=\textwidth]{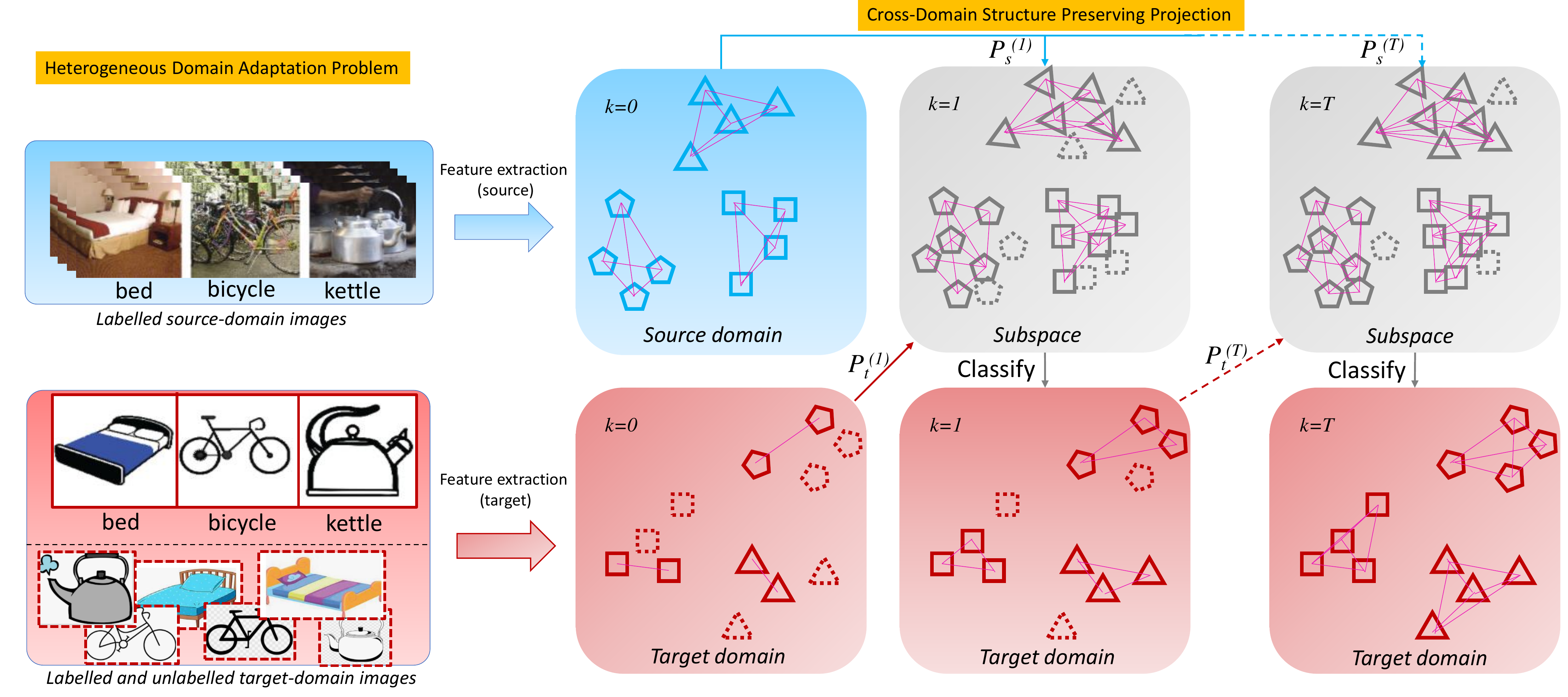}}
    {\caption{An illustration of the heterogeneous domain adaptation problem and our proposed approach using cross-domain structure preserving projection. Left: the HDA problem aims at recognizing unlabelled target-domain samples with the access of labelled source-domain samples and limited labelled target-domain samples. Right: The red and the blue colours are used to represent the feature vectors of samples in the target and source domains respectively; markers of different shapes represent samples from different classes; dashed markers represent unlabelled samples; our proposed CDSPP iteratively learn a common subspace in which the unlabelled target-domain samples are pseudo-labelled and selectively added to the training data set to promote the subspace learning in the next iteration.}
        \label{fig:framework}}
\end{figure}
\subsection{Locality Preserving Projection} \label{sec:lpp}
To make the paper self-contained, we briefly describe the original LPP algorithm \citep{he2004locality} before introducing our proposed CDSPP in the next subsection. Locality Preserving Projection (LPP) was proposed by \citet{he2004locality} to learn a favourable subspace where the local structures of data in the original feature space can be well preserved. Suppose $\bm{x}_i \in \mathbb{R}^{d_0}$ and $\bm{x}_j\in \mathbb{R}^{d_0}$ are two data points in the original feature space, LPP aims at learning a projection matrix $\bm{P} \in \mathbb{R}^{d\times d_0}$ ($d<<d_0$) so that data points close to each other in the original space will still be close in the projected subspace. The objective of LPP can be formulated as:
\begin{equation}
    \label{eq:lpp}
    \min_{\bm{P}} \sum_{i,j} ||\bm{P}^T \bm{x}_i - \bm{P}^T \bm{x}_j||_2^2 \bm{W}_{ij},
\end{equation}
where $\bm{W}$ is the adjacency matrix of the graph constructed by all the data points. According to \cite{he2004locality}, the edges of the graph can be created by either $\epsilon-$neighbourhoods or $k$-nearest neighbours. The edge weights can be determined by the heat kernel $W_{ij} = e^{-\frac{||\bm{x}_i-\bm{x}_j||^2}{t}}$ or the simple binary assignment (i.e. all edges have the weights of 1).
Note that LPP is an unsupervised learning method without the need for labelling information. In the following subsection, we will describe how to extend the LPP algorithm to solve the HDA problems where there exist two heterogeneous domains and a mixture of labelled and unlabelled data.

\subsection{Cross-Domain Structure Preserving Projection} \label{sec:CDSPP}
The supervised version of LPP \citep{wang2017zero} was proved to be able to learn a subspace of better separability than other dimensionality reduction algorithms such as Linear Discriminant Analysis (LDA) \citep{wang2019unifying}. One limitation of LPP is that it can only learn the subspace from samples represented in a homogeneous feature space. To address this problem, we extend the traditional LPP so that its favourable characteristics can benefit cross-domain common subspace learning. Specifically, we aim to learn a projection matrix $\bm{P}_s \in \mathbb{R}^{d_s \times d}$ for the source domain and a projection matrix $\bm{P}_t \in \mathbb{R}^{d_t \times d}$ for the target domain to project the samples from source and target domains into a common subspace whose dimensionality is $d$. We expect the samples projections are close to one another if they are from the same class regardless of which domain they are from. To this end, we have the following objective:
\begin{equation}
\label{eq:cost}
\begin{array}{ll}
\displaystyle \min_{\bm{P}_s,\bm{P}_t} &(\sum_{i,j}^{n_s} || \bm{P}_s^T \bm{x}_i^s - \bm{P_s}^T \bm{x}_j^s||_2^2 \bm{W}_{ij}^s \\ 
\displaystyle &  +\sum_i^{n_s} \sum_j^{n_t} || \bm{P}_s^T \bm{x}_i^s - \bm{P}_t^T \bm{x}_j^t||_2^2 \bm{W}_{ij}^c \\ 
\displaystyle & +\sum_{i,j}^{n_t} || \bm{P}_t^T \bm{x}_i^t - \bm{P}_t^T \bm{x}_j^t||_2^2 \bm{W}_{ij}^t)
\end{array}
\end{equation}
where $\bm{P}^T$ is the transpose of $\bm{P}$; $\bm{W}^s \in \mathbb{R}^{n_s \times n_s}$ is the similarity matrix of the source samples and $\bm{W}^s_{ij} = 1$ if $y^s_i = y^s_j$, 0 otherwise. Similary, $\bm{W}^t \in \mathbb{R}^{n_t \times n_t}$ is the similarity matrix of the \textit{labelled} target samples and $\bm{W}^t_{ij} = 1$ if $y^t_i = y^t_j$, 0 otherwise. $\bm{W}^c \in \mathbb{R}^{n_s \times n_t}$ is the cross-domain similarity matrix and $\bm{W}^c_{ij} = 1$ if $y^s_i = y^t_j$, 0 otherwise. It is noteworthy that all the feature vectors are $l2$-normalised to get rid of the effect of different magnitudes across features. This pre-processing has been proved to be useful for common subspace learning in \cite{wang2017zero,wang2019unifying,wang2020unsupervised}. 

\begin{proposition}
The objective in Eq.(\ref{eq:cost}) can be reformulated as follows:
\begin{equation}
    \label{eq:costDiv}
    \max_{\bm{P}_s,\bm{P}_t} \frac{tr({\bm{X}^s}^T \bm{P}_s \bm{P}_t^T \bm{X}^t {\bm{W}^c}^
    T)}{tr({\bm{X}^s}^T \bm{P}_s \bm{P}_s^T \bm{X}^s \bm{L}^s) +  tr({\bm{X}^t}^T \bm{P}_t \bm{P}_t^T \bm{X}^t \bm{L}^t)}
\end{equation}
where $\bm{L}^s = \bm{D}^s - \bm{W}^s + \frac{1}{2}\bm{D}^{cs}$ and $\bm{L}^t = \bm{D}^t - \bm{W}^t + \frac{1}{2}\bm{D}^{ct}$; $\bm{D}^s \in \mathbb{R}^{n_s \times n_s}$ is a diagonal matrix with $\bm{D}^s_{ii} = \sum_j^{n_s} \bm{W}^s_{ij}$ and $\bm{D}^t \in \mathbb{R}^{n_t \times n_t}$ is a diagonal matrix with $\bm{D}^t_{jj} = \sum_i^{n_t} \bm{W}^t_{ij}$; $\bm{D}^{cs} \in \mathbb{R}^{n_s\times n_s}$ is a diagonal matrix with $\bm{D}_{ii}^{cs} = \sum_j^{n_t} \bm{W}^c_{ij}$ and $\bm{D}^{ct} \in \mathbb{R}^{n_t\times n_t}$ is a diagonal matrix with $\bm{D}_{jj}^{ct} = \sum_i^{n_s} \bm{W}^c_{ij}$.
\end{proposition}

\begin{proof}
By firstly doing the binomial expansion then transforming it to the form of matrix multiplication and trace of matrices, the first term in Eq.(\ref{eq:cost}) can be reformulated as follows: 
\begin{equation}
    \label{eq:term1}
    \begin{array}{ll}
    \sum_{i,j}^{n_s} || \bm{P}_s^T \bm{x}_i^s - \bm{P_s}^T \bm{x}_j^s||_2^2 \bm{W}_{ij}^s \\ 
     = \sum_{i,j}^{n_s} ({\bm{x}_i^s}^T \bm{P}_s \bm{P}_s^T \bm{x}_i^s - 2 {\bm{x}_i^s}^T \bm{P}_s \bm{P}_s^T \bm{x}_j^s + {\bm{x}_j^s}^T \bm{P}_s \bm{P}_s^T \bm{x}_j^s) \bm{W}^s_{ij}\\
     = 2 \sum_i^{n_s} {\bm{x}_i^s}^T \bm{P}_s \bm{P}_s^T \bm{x}_i^s \bm{D}_{ii}^s - 2\sum_{i,j}^{n_s} {\bm{x}_i^s}^T \bm{P}_s \bm{P}_s^T \bm{x}_j^s \bm{W}^s_{ij} \\
    = 2 tr({\bm{X}^s}^T \bm{P}_s \bm{P}_s^T \bm{X}^s \bm{D}^s) - 2 tr({\bm{X}^s}^T \bm{P}_s \bm{P}_s^T \bm{X}^s \bm{W}^s) \\
    \end{array}
\end{equation}

In the similar way, the third term in Eq.(\ref{eq:cost}) can be rewritten as:
\begin{equation}
    \label{eq:term3}
    \begin{array}{ll}
        \sum_{i,j}^{n_t} || \bm{P}_t^T \bm{x}_i^t - \bm{P}_t^T \bm{x}_j^t||_2^2 \bm{W}_{ij}^t \\
        = 2 tr({\bm{X}^t}^T \bm{P}_t \bm{P}_t^T \bm{X}^t \bm{D}^t) - 2 tr({\bm{X}^t}^T \bm{P}_t \bm{P}_t^T \bm{X}^t \bm{W}^t)
    \end{array}
\end{equation}

The second term in Eq.(\ref{eq:cost}) can be rewritten as:
\begin{equation}
    \label{eq:term2}
    \begin{array}{ll}
    \sum_i^{n_s} \sum_j^{n_t} || \bm{P}_s^T \bm{x}_i^s - \bm{P}_t^T \bm{x}_j^t||_2^2 \bm{W}_{ij}^c  \\
     = \sum_i^{n_s} \sum_j^{n_t} ({\bm{x}_i^s}^T \bm{P}_s \bm{P}_s^T \bm{x}_i^s - 2 {\bm{x}_i^s}^T \bm{P}_s \bm{P}_t^T \bm{x}_j^t \\
      \qquad\qquad\qquad + {\bm{x}_j^t}^T \bm{P}_t \bm{P}_t^T \bm{x}_j^t) \bm{W}^c_{ij}\\
     = \sum_i^{n_s} {\bm{x}_i^s}^T \bm{P}_s \bm{P}_s^T \bm{x}_i^s \bm{D}_{ii}^{cs} - 2\sum_i^{n_s}\sum_j^{n_t} {\bm{x}_i^s}^T \bm{P}_s \bm{P}_t^T \bm{x}_j^t \bm{W}^c_{ij} \\
     \qquad\qquad\qquad + \sum_j^{n_t} {\bm{x}_j^t}^T \bm{P}_t \bm{P}_t^T \bm{x}_j^t \bm{D}_{jj}^{ct}  \\
    = tr({\bm{X}^s}^T \bm{P}_s \bm{P}_s^T \bm{X}^s \bm{D}^{cs}) - 2 tr({\bm{X}^s}^T \bm{P}_s \bm{P}_t^T \bm{X}^t {\bm{W}^c}^T) \\
    \qquad\qquad\qquad + tr({\bm{X}^t}^T \bm{P}_t \bm{P}_t^T \bm{X}^t \bm{D}^{ct})\\
    \end{array}
\end{equation}
Substitute Eqs.(\ref{eq:term1}-\ref{eq:term2}) into the objective Eq.(\ref{eq:cost}), we have the following form of objective:
\begin{equation}\label{eq:costMatrixForm}
\begin{array}{ll}
    \displaystyle \min_{\bm{P}_s,\bm{P}_t} \big(tr({\bm{X}^s}^T \bm{P}_s \bm{P}_s^T \bm{X}^s \bm{L}^s) +  tr({\bm{X}^t}^T \bm{P}_t \bm{P}_t^T \bm{X}^t \bm{L}^t) \\
    - tr({\bm{X}^s}^T \bm{P}_s \bm{P}_t^T \bm{X}^t {\bm{W}^c}^T)\big)
\end{array}
\end{equation}
where $\bm{L}^s = \bm{D}^s - \bm{W}^s + \frac{1}{2}\bm{D}^{cs}$ and $\bm{L}^t = \bm{D}^t - \bm{W}^t + \frac{1}{2}\bm{D}^{ct}$.

Minimizing the objective in Eq.(\ref{eq:costMatrixForm}) is equivalent to maximizing the objective in Eq.(\ref{eq:costDiv}).

\end{proof}

\begin{proposition}
\label{propEigen}
The objective in Eq.(\ref{eq:costDiv}) is equivalent to the following generalized eigenvalue problem and the optimal projection matrix $\bm{P}=\begin{bmatrix}\bm{P}_s\\ \bm{P}_t \end{bmatrix}$ can be formed by $d$ eigenvectors corresponding to the largest $d$ eigenvalues:
\begin{equation}
    \label{eq:eig}
    \bm{A} \bm{P} = (\bm{B}+\alpha \bm{I}) \bm{P}\Lambda 
\end{equation}
where $\bm{I} \in \mathbb{R}^{(n_s+n_t)\times(n_s+n_t)}$ is an identity matrix, $\alpha$ is a hyper-parameter for regularization \citep{wang2017zero}, $\Lambda$ is a diagonal eigenvalue matrix and
\begin{gather}\label{eq:a}
 \bm{A} = \begin{bmatrix} \bm{0} & \bm{X}^s\bm{W}^c{\bm{X}^t}^T \\ \bm{X}^t{\bm{W}^c}^T {\bm{X}^s}^T & \bm{0} \end{bmatrix},
 \end{gather}

\begin{gather}\label{eq:b}
    \bm{B} = 
 \begin{bmatrix} \bm{X}^s \bm{L}^s {\bm{X}^s}^T & \bm{0} \\ \bm{0} & \bm{X}^t \bm{L}^t {\bm{X}^t}^T \end{bmatrix}.
 \end{gather}
\end{proposition}

\begin{proof}
To make the proof process concise, we introduce notations $\bm{S}_s=\bm{X}^s \bm{L}^s {\bm{X}^s}^T$, $\bm{S}_t=\bm{X}^t \bm{L}^t {\bm{X}^t}^T$ and $\bm{S}_c=\bm{X}^s\bm{W}^c{\bm{X}^t}^T$. 

Let 
\begin{equation}
    \label{eq:costJ}
    \mathcal{J}(\bm{P}_s,\bm{P}_t) = \frac{tr(\bm{P}_t^T \bm{S}_c^T \bm{P}_s)}{tr(\bm{P}_s^T \bm{S}_s \bm{P}_s)+tr(\bm{P}_t^T \bm{S}_t \bm{P}_t)}
\end{equation}
 be the objective function in Eq.(\ref{eq:costDiv}), we calculate the partial derivatives \citep{petersen2008matrix} of $\mathcal{J}$ w.r.t. $\bm{P}_s$ and $\bm{P}_t$ respectively, set them to 0 and get the following equations:

\begin{equation}
    \label{eq:partialPs}
    \bm{S}_c \bm{P}_t =\frac{2 tr(\bm{P}_t^T \bm{S}_c \bm{P}_s)}{tr(\bm{P}_s^T \bm{S}_s \bm{P}_s)+tr(\bm{P}_t^T \bm{S}_t \bm{P}_t)} \bm{S}_s \bm{P}_s
\end{equation}
\begin{equation}
    \label{eq:partialPt}
    \bm{S}_c^T \bm{P}_s =\frac{2 tr(\bm{P}_t^T \bm{S}_c \bm{P}_s)}{tr(\bm{P}_s^T \bm{S}_s \bm{P}_s)+tr(\bm{P}_t^T \bm{S}_t \bm{P}_t)} \bm{S}_t \bm{P}_t
\end{equation}
Note that the coefficients on the right side of Eqs(\ref{eq:partialPs}-\ref{eq:partialPt}) are exactly the objective in Eq.(\ref{eq:costJ}). It is easy to construct the following generalized eigenvalue problem by combining Eqs.(\ref{eq:partialPs}-\ref{eq:partialPt}): 
\begin{gather}
 \begin{bmatrix} \bm{0} & \bm{S}_c \\ \bm{S}_c^T & \bm{0} \end{bmatrix} 
 \begin{bmatrix} \bm{P}_s \\ \bm{P}_t \end{bmatrix}  = 
 \begin{bmatrix} \bm{S}_s & \bm{0} \\ \bm{0} & \bm{S}_t \end{bmatrix} 
  \begin{bmatrix} \bm{P}_s \\ \bm{P}_t \end{bmatrix} \Lambda.
 \end{gather}
The maximum objective is given by the largest eigenvalue solution to the generalized eigenvalue problem \citep{he2004locality} and the eigenvectors corresponding to the largest $d$ eigenvalues will form the projection matrix $\bm{P}_s$ and $\bm{P}_t$.

\end{proof}

\subsection{Recognition in the Subspace}\label{sec:recognition}
Once the projection matrices $\bm{P}_s$ and $\bm{P}_t$ are learned, we are able to project all the labelled samples into the learned common subspace by $\bm{z}^s_i = \bm{P}_s^T \bm{x}^s_i$ and $\bm{z}^t_i = \bm{P}_{t}^T \bm{x}^t_i$. Similar to the pre-processing for the training data, the feature vectors $\bm{x}$ need to be $l2$-normalised before being projected to the subspace. For the same reason, we also apply $l2$-normalisation to the projected vectors $\bm{z}$. The $l2$-normalisation re-allocates the projected vectors in the subspace to the surface of a hyper-sphere which will benefit the measurement of distances when do the recognition using the nearest neighbour method. More importantly, the $l2$-normalisation adds non-linearity to the process so that our proposed CDSPP method can handle practical problems when linear projection assumptions do not hold.

For each class, we calculate the class mean $\bar{\bm{z}}_c$ for $c=1,2,...,C$ using all the labelled sample from both source and target domains. Given an unlabelled target sample $\bm{x}^u$, we classify it to the closest class in terms of its Euclidean distances to the class means:
\begin{equation}
    \label{eq:recognition}
    y^* = \argmin_c d(\bar{\bm{z}}_c, \bm{P}_t^T \bm{x}^u)
\end{equation}
The proposed CDSPP for supervised HDA is summarized in Algorithm \ref{alg:hdasup}.

\textbf{Relation to DAMA} The CDSPP algorithm is quite similar to DAMA proposed in \citep{wang2011heterogeneous} at the first glance, however, they are essentially different from each other in that CDSPP does not seek to push the sample projections belonging to different classes apart, since the penalty imposed for this purpose (e.g., maximizing the term $B$ in \citep{wang2011heterogeneous}) might misguide the solution to focus too much on the separation of classes which are originally close to each other and hurt the overall separability of the learned subspace. In contrast, our objective in Eq.(\ref{eq:cost}) can guarantee the separability of the learned subspace by promoting the preserving of cluster structures underlying the original data distributions, which is simpler but more effective as validated by experiments.

\begin{algorithm}
\caption{Supervised HDA using CDSPP}
    \label{alg:hdasup}
    \renewcommand{\algorithmicinput}{\textbf{Input:}}
    \renewcommand{\algorithmicoutput}{\textbf{Output:}}
    \renewcommand{\algorithmicensure}{\textbf{Training:}}
    \renewcommand{\algorithmicrequire}{\textbf{Testing:}}
    \begin{algorithmic}[1]
        \INPUT labelled source data set $\mathcal{D}^s = \{(\bm{x}^s_i,y^s_i)\}, i = 1,2,...,n_s$ and labelled target data set $\mathcal{D}^t=\{\bm{x}_i^t,y_i^t\},i=1,2,...,n_t$, the dimensionality of subspace $d$.
        
        \OUTPUT The projection matrix $\bm{P}_s$ and $\bm{P}_t$ for source and target domains, the labels predicted for unlabelled target data $\bm{X}^u$.
        \ENSURE
        \STATE Learn the projection $\bm{P}_s$ and $\bm{P}_t$ using labelled data $\mathcal{D}^s \cup \mathcal{D}^t$ by solving the generalized eigenvalue problem in Eq.(\ref{eq:eig});
        \REQUIRE
        \STATE Classify unlabelled target samples $\bm{X}^u$ using Eq.(\ref{eq:recognition}).
    \end{algorithmic}
\end{algorithm}

\begin{algorithm}
\caption{Semi-supervised HDA using CDSPP}
    \label{alg:hda}
    \renewcommand{\algorithmicinput}{\textbf{Input:}}
    \renewcommand{\algorithmicoutput}{\textbf{Output:}}
    \renewcommand{\algorithmicensure}{\textbf{Training:}}
    \renewcommand{\algorithmicrequire}{\textbf{Testing:}}
    \begin{algorithmic}[1]
        \INPUT labelled source data set $\mathcal{D}^s = \{(\bm{x}^s_i, y^s_i)\}, i = 1,2,...,n_s$, labelled target data set $\mathcal{D}^t=\{\bm{x}_i^t, y^t_i\},i=1,2,...,n_t$, unlabelled target data set $\mathcal{D}^u=\{\bm{x}_i^u\},i=1,2,...,n_u$ the dimensionality of subspace $d$, number of iteration $T$.
        
        \OUTPUT The projection matrix $\bm{P}_s$ and $\bm{P}_t$ for source and target domains, the labels predicted for unlabelled target data $\bm{X}^u$.
        \ENSURE
        \STATE Initialize $k$=1;
        \STATE Learn the projection $\bm{P}_s^{(0)}$ and $\bm{P}_t^{(0)}$ using labelled data $\mathcal{D}^s \cup \mathcal{D}^t$ by solving the generalized eigenvalue problem in Eq.(\ref{eq:eig});
        \STATE Get the unlabelled target data set $\mathcal{D}^u$;
        \WHILE {$k \leq T$}
        \STATE Label all the samples from $\mathcal{D}^u$ by Eq.(\ref{eq:recognition});
        \STATE Select a subset of (top $kn_u/T$ most confident) pseudo-labelled target samples $\mathcal{S}^{(k)} \subseteq \mathcal{D}^u$;
        \STATE Learn $\bm{P}_s^{(k)}$ and $\bm{P}_t^{(k)}$ using a combination of labelled and pseudo-labelled data sets $\mathcal{D}^s \cup \mathcal{D}^t \cup \mathcal{S}^{(k)}$;
        \STATE $k \leftarrow k+1$;
        \ENDWHILE
        
        \REQUIRE
        \STATE Classify unlabelled target samples $\bm{X}^u$ using Eq.(\ref{eq:recognition}).
    \end{algorithmic}
\end{algorithm}

\subsection{Extending to Semi-Supervised HDA} \label{sec:shda}
The CDSPP algorithm is naturally suitable for supervised HDA but can be extended to semi-supervised HDA by incorporating it into an iterative pseudo-labelling framework \citep{wang2019unifying}. Given a set of unlabelled target samples $\bm{X}^u$, they can be labelled by Eq.(\ref{eq:recognition}). The pseudo-labelled target samples can be used to update the projection matrices $\bm{P}_s$ and $\bm{P}_t$. However, when the domain shift is large and the number of labelled target samples is limited, the pseudo-labels can be wrong for a considerable number of target samples. In this case, the mistakenly pseudo-labelled target samples might hurt projection learning. To reduce this risk, confidence aware pseudo-labelling is proposed in \citep{wang2019unifying}. We employ the same idea and progressively select the most confidently pseudo-labelled target samples for the next iteration of CDSPP learning. The proposed CDSPP for semi-supervised HDA is summarized in Algorithm \ref{alg:hda}.

\subsection{Complexity Analysis}
The time complexity of CDSPP is mainly contributed by two parts: the matrix multiplications in Eqs.(\ref{eq:a}-\ref{eq:b}) and the eigen decomposition problem. The complexity of matrix multiplications is $\mathcal{O}((n_s+n_t)d_sd_t)$. The complexity of eigen decomposition is generally $\mathcal{O}((d_s+d_t)^3)$. As a result, the CDSPP algorithm has a complexity of $\mathcal{O}((n_s+n_t)d_sd_t+(d_s+d_t)^3)$. In the case of semi-supervised HDA, the time complexity will be increased by $T$ times and the value of $n_t$ increases by the number of selected pseudo-labelled target samples in each iteration.

\section{Experiments} \label{sec:experiments}
To evaluate the effectiveness of the proposed method in heterogeneous domain adaptation, we conduct thorough experiments on commonly used benchmark datasets. We compare the proposed approach with existing HDA methods and analyze its sensitivity to hyper-parameters. 

\subsection{Datasets and Experimental Settings}
\textbf{Office-Caltech} \citep{gong2012geodesic} is an image dataset containing four domains: Amazon (A), Webcam (W), DSLR (D) and Caltech (C) from 10 common classes. Two image features (i.e. 4096-dim Decaf$_6$ and 800-dim SURF) are used for cross-domain adaptation.
\textbf{Multilingual Reuters Collection (MRC)} \citep{amini2009learning} is a cross-lingual text classification dataset containing 6 classes in 5 languages (i.e. EN, FR, GE, IT, SP). We follow the settings in \citep{hubert2016learning} extracting BoW features and applying PCA to get heterogeneous feature dimensions (i.e. 1131, 1230, 1417, 1041, 807 respectively) for five domains. In our experiments, SP serves as the target domain and the other four languages as the source domains respectively. As a result, we have four HDA tasks. 
\textbf{NUS-WIDE} \citep{chua2009nus} and \textbf{ImageNet} \citep{deng2009imagenet} datasets are employed for text to image domain adaptation. Following \citep{chen2016transfer} we consider 8 overlapping classes using tag information represented by 64-dim features from NUS-WIDE as the source domain and 4096-dim Decaf$_6$ features of images from ImageNet as the target domain. 
However, the above datasets contain very limited numbers of classes and may not discriminate capabilities of different methods.  We introduce \textbf{Office-Home} \citep{venkateswara2017deep} containing four domains (i.e. Art, Clipart, Product and Real-world)  as a new testbed for HDA. We use VGG16 \citep{simonyan2014very} and ResNet50 \citep{he2016deep} models pre-trained on ImageNet to extract 4096-dim and 2048-dim features. More details of the datasets and protocols used in our experiments are summarized in Table \ref{table:datasets}.

\begin{table}[!t]
    \centering
    {
        \centering
        \caption[]{The statistics of datasets (notations: LSS/c -- labelled Source Samples per class; LTS/c -- labelled Target Sample per class; UTS/c -- Unlabelled Target Samples per class; all -- all samples except the ones chosen as labelled target samples).
        }
        \label{table:datasets}
        \resizebox{0.8\columnwidth}{!}{%
            \begin{tabular}{ccccccc}
                \hline
                Dataset & \# Domain & \# Task & \# Class & \shortstack[c]{\# LSS/c} & \shortstack[c]{\# LTS/c} & \shortstack[c]{\# UTS/c} \\ \hline
                Office-Caltech & 4 & 16 & 10 & 20 & 3 & all \\
                MRC & 5 & 4 & 6 & 100 & 10 & 500\\
                NUS-ImageNet & 2 & 1 & 8 & 100 & 3 & 100\\
                Office-Home & 4 & 16 & 65 & 20 & 3 & all\\
                \hline
            \end{tabular}%
        }
    }
\end{table}

\subsection{Comparative Methods}

To evaluate the effectiveness of the proposed CDSPP in different HDA problems, we conduct a comparative study and compare the performance of CDSPP with state-of-the-art methods in both supervised and semi-supervised settings. Specifically, we compare with SVM$_t$, HFA \citep{li2013learning}, CDLS$\_$sup \citep{hubert2016learning} and a variant of DAMA \cite{wang2011heterogeneous} under the supervised HDA setting (i.e. the unlabelled target samples are not available during training). 
\begin{itemize}
    \item {SVM$_t$} is a baseline method that trains an SVM model on the target dataset $\mathcal{D}^t$ in a conventional supervised learning manner without using the source domain data.
    \item HFA (Heterogeneous Feature Augmentation \citep{li2013learning}) is designed to solve the supervised HDA problem by augmenting the original features $\bm{x}^s, \bm{x}^t$ with transformed features $\bm{P}\bm{x}^s$, $\bm{Q}\bm{x}^t$ and zero vectors. The projection matrices $\bm{P}$ and $\bm{Q}$ for the source and target domains map the original features into a common subspace so that the similarity of features across two domains can be directly compared. The objective of learning $\bm{P}$ and $\bm{Q}$ is incorporated into the framework of classifier (i.e. SVM) training.
    \item CDLS$\_sup$ (Cross-Domain Landmark Selection \citep{hubert2016learning}) is the supervised version of CDLS which aims to learn a projection matrix $\bm{A}$ to map source-domain data into the target domain. The objective is to align the cross-domain marginal and conditional data distributions by minimizing the  Maximum Mean Discrepancy (MMD).
    \item DAMA$\_sup$ (Domain Adaptation Using Manifold Alignment \cite{wang2011heterogeneous}) is originally designed for semi-supervised HDA problems. Similar to our proposed CDSPP, it also aims to learn two projection matrices to map source and target domain data to a common subspace where the manifolds of data from two domains are aligned. We adapt it for supervised HDA by considering only labelled data when constructing the feature similarity matrix $\bm{W}$, the label based similarity matrix $\bm{W}^s$ and dissimilarity matrix $\bm{W}^d$. Different from the suggestion in the original paper, we use an optimal $\mu=0.1$ throughout our experiments since this setting achieves the best performance.
    
\end{itemize}

For semi-supervised HDA, we compare with DAMA \citep{wang2011heterogeneous}, SHFA \citep{li2013learning}, CDLS \citep{hubert2016learning}, PA \citep{li2018heterogeneous}, TIT \citep{li2018transfer}, STN \citep{yao2019heterogeneous}, DDACL\citep{yao2020discriminative}, SSAN \citep{li2020simultaneous}
and DAMA+, our extension of DAMA by incorporating it into our iterative learning framework (c.f. Section \ref{sec:shda}). 

\begin{itemize}
    \item DAMA \cite{wang2011heterogeneous} is employed in the semi-supervised HDA experiments in its original form except the hyper-parameter $\mu$ is set as 0.1 as our experimental results show empirically it gives the optimal performance.
    \item SHFA (Semi-supervised HFA \cite{li2013learning}) is an extension of HFA. It takes advantage of the unlabelled target-domain data by replacing the SVM in HFA with a Transductive SVM (T-SVM) \cite{collobert2006large} model.
    \item CDLS \cite{hubert2016learning} is designed for semi-supervised HDA. As described above, it aims to learn a projection matrix $\bm{A}$ to map source-domain data into the target domain so that cross-domain data can be aligned. When unlabelled target-domain data are available in the semi-supervised HDA, the unlabelled data are pseudo-labelled by the supervised version CDLS$\_sup$. Subsequently, the pseudo-labelled data are used to update the projection $\bm{A}$. The processes are repeated for multiple iterations. In particular, the instances are weighted by learnable weights when constructing the objective function.
    \item PA (Progressive Alignment \cite{li2018heterogeneous}) and TIT (Transfer Independently Together \cite{li2018transfer}) share a similar framework to CDLS but employ different algorithms of transformation matrix learning (involving MMD, graph embedding and regularisation) and different instance weight estimation strategies. The unlabelled target-domain data are also pseudo-labelled to optimize the transformation matrices iteratively.
    \item STN (Soft Transfer Network \cite{yao2019heterogeneous}) jointly learns a domain-shared classifier and a domain-invariant subspace in an end-to-end manner. The network model is learned by optimising the objective similar to those in the aforementioned works, i.e., MMD. Besides, the unlabelled target-domain data are used by the soft-label strategy.
    \item DDACL (Discriminative Distribution Alignment with Cross-entropy Loss \cite{yao2020discriminative}) trains an adaptive classifier by both reducing the distribution divergence and enlarging distances between class centroids.
    
    \item SSAN (Simultaneous Semantic Alignment Network \cite{li2020simultaneous}) employs an implicit semantic correlation loss to transfer the correlation knowledge of source categorical prediction distributions to the target domain. A triplet-centroid alignment mechanism is explicitly applied to align feature representations for each category by leveraging target pseudo-labels. Note that the results of best accuracy of the test samples throughout the training process were reported in \cite{li2020simultaneous}, we argue that this is not achievable in practice since the labels of test samples are not available during training. Instead, we report the results achieved in the last iterations in our experiments.
    
    \item DAMA+ is our adaptation of the original DAMA by incorporating the DAMA algorithm into our proposed iterative learning framework with selective pseudo-labelling. Specifically, we use the supervised version of DAMA described above to initialise the projection matrices and get the pseudo-labels of unlabelled target-domain data. The selected most confidently pseudo-labelled target-domain data will contribute to the update of projection matrices in the next iteration of learning. Finally, the optimal projection matrices and predicted target-domain data labels are obtained. 
    \item CDSPP+PCA is a variant of CDSPP by applying PCA to the original features and CDSPP is subsequently applied to the low-dimensional features. This pre-processing is specially designed for handcrafted features in the MRC and NUS-ImageNet datasets and 50 principal components are reserved for all features. 
\end{itemize}
In all experiments, we use the optimal parameters suggested in the original papers for the comparative methods if not otherwise specified whilst set the hyper-parameters of CDSPP empirically as $d$ equal to the number of classes in the dataset, $\alpha=10$ and $T=5$. More details of hyper-parameter value selection will be discussed later.

\subsection{Comparison Results}

Although there exist fixed experimental protocols in terms of the number of labelled samples used for training as shown in Table \ref{table:datasets}, there is no standard data splits publicly available to follow. As will be demonstrated in our experimental results, selecting different samples for training can lead to significant performance variance.  We generate data splits randomly in our experiments\footnote{The data splits and code are released: https://github.com/hellowangqian/cdspp-hda}. To mitigate the biases caused by the data selection, ten random data splits are generated for each adaptation task. We report the mean and standard deviation of the classification accuracy over these ten trials for each adaptation task. The results for all comparative methods are reproduced using the same data splits for a direct comparison. The implementations released by the authors are employed in our experiments. As a result, the results in this paper are not comparable with those reported in other papers since different sample selections have been used in our experiments. Our experimental results of both supervised and semi-supervised HDA on four datasets are shown in Tables \ref{table:mrc-tag2image}-\ref{table:resnet2vgg} from which we can obtain the following insights.

\begin{table}[!t]
    \centering
    {
        \centering
        \caption[]{Mean(std) of classification accuracy (\%) over ten trials for cross-language and tag-to-image adaptation under supervised (denoted by $*$) and semi-supervised settings (each column represents one Source $\to$ Target adaptation task).
        }
        \label{table:mrc-tag2image}
        \resizebox{0.9\columnwidth}{!}{%
            \begin{tabular}{lccccc|c}
                \hline
                Method & EN$\to$SP & FR$\to$SP& GE$\to$SP& IT$\to$SP & Avg & Tag$\to$Image \\ \hline
                SVM$_t$ * & 67.0(2.4) & 67.0(2.4) & 67.0(2.4) & 67.0(2.4) & 67.0 & 60.6(6.0) \\
                HFA \citep{li2013learning} * & 68.1(3.0) & 68.0(3.0) & 68.0(3.0) & 68.0(3.0) & 68.0 & 67.5(2.5)\\
                CDLS\_sup \citep{hubert2016learning} * & 63.0(3.6) & 63.4(2.4) & 64.0(2.2) & 64.6(3.6) & 63.8 & 66.3(3.9)\\
                DAMA\_sup *  & 66.8(2.5) & 66.3(3.3) & 66.3(3.0) & 66.7(2.7) & 66.5 & 66.9(2.6) \\
                CDSPP\_sup (Ours) * & 67.2(2.8) & 67.3(2.9) & 67.3(2.9) & 67.3(2.8) & 67.3 & 67.2(3.0)\\
                \hline
                DAMA \citep{wang2011heterogeneous}  & 67.0(2.5) & 66.6(3.1) & 66.7(3.0) & 67.4(2.8) & 66.9 & 67.0(2.5)\\
                SHFA \citep{li2013learning}& 66.9(3.7) & 66.1(2.7) & 67.5(3.1) & 67.4(2.2) & 67.0 &  68.1(2.7)\\
                CDLS \citep{hubert2016learning} & 69.4(3.0) & 69.4(3.0) & 69.4(3.2) & 69.3(3.1) & 69.4 & 69.6(2.1)\\
                PA \citep{li2018heterogeneous}& \bf 71.4(2.9) & \bf 71.6(2.9) & \bf 71.7(3.0) & \bf 72.3(2.5) & \bf 71.7 & 70.5(4.0)\\
                TIT \citep{li2018transfer} & 67.1(2.8) & 67.6(2.6) & 66.1(3.5) & 67.8(2.0) & 67.2 & 70.7(3.4)\\
                STN \citep{yao2019heterogeneous} & 67.1(3.6) & 67.3(2.5) & 66.9(3.5) & 66.7(3.8) & 67.0 & 74.3(5.2)\\
                DDACL \cite{yao2020discriminative} & 70.2(3.0) & 70.4(3.1) & 70.8(3.0) & 70.9(3.0) & 70.6 & 73.8(2.8) \\
                SSAN \cite{li2020simultaneous}& 69.9(2.9)& 69.4(2.8)& 69.3(4.0)& 70.2(2.5)& 69.7 & 71.4(1.2) \\
                DAMA + & 68.9(2.1) & 68.8(4.0) & 68.9(2.7) & 68.2(3.5) & 68.7 & 73.4(4.3)\\
                CDSPP (Ours) & 69.1(3.2) & 69.0(3.6) & 68.8(3.2) & 68.8(3.0) & 68.9 & 74.7(3.4)\\
                CDSPP+PCA (Ours) & \bf 71.2(3.2) & \bf 71.7(3.1) & \bf 71.4(3.0) & \bf 72.1(3.0) & \bf 71.6 & \bf 76.5(3.3) \\

                \hline
            \end{tabular}%
        }
    }
\end{table}

\begin{table*}[!htbp]
    \centering
    {
        \centering
        \caption[]{Mean(std) of classification accuracy (\%) over ten trials on the Office-Caltech dataset using SURF (source) and Decaf$_6$ (target) features under supervised (denoted by $*$) and semi-supervised settings (each column represents one Source $\to$ Target adaptation task).}
        \label{table:surf2decaf}
        \resizebox{\columnwidth}{!}{%
            \begin{tabular}{l cccc cccc ccccc cccc}
                \hline
                Method & C$\to$C & C$\to$A & C$\to$D & C$\to$W&A$\to$C & A$\to$A&A$\to$D & A$\to$W & D$\to$C & D$\to$A & D$\to$D & D$\to$W & W$\to$C & W$\to$A & W$\to$D & W$\to$W& Avg \\ \hline
                SVM$_t$ * & 73.6(4.9) & 87.9(2.2) & 92.3(3.6) & 88.4(3.8)& 73.6(4.9) & 87.9(2.2) & 92.3(3.6) & 88.4(3.8)& 73.6(4.9) & 87.9(2.2) & 92.3(3.6) & 88.4(3.8)& 73.6(4.9) & 87.9(2.2) & 92.3(3.6) & 88.4(3.8)&85.5\\
                HFA \citep{li2013learning} * & 80.1(2.3) & 88.9(1.9) & 91.6(3.6) & 90.7(3.5) & 80.2(2.3) & 88.9(1.9) & 91.5(3.6) & 90.5(3.6) & 80.2(2.2) & 88.8(1.9) & 91.8(3.6) & 90.7(3.5) & 80.2(2.3) & 88.8(1.9) & 91.5(3.7) & 90.6(3.7) & 87.8\\
                CDLS\_sup \citep{hubert2016learning} *& 76.1(2.1) & 86.6(3.2) & 91.3(4.7) & 87.4(3.5) & 75.9(3.5) & 87.0(2.8) & 90.6(3.8) & 86.0(3.6) & 51.5(4.4) & 74.2(2.4) & 86.6(3.2) & 77.2(5.1) & 74.7(4.1) & 85.4(3.0) & 90.5(3.8) & 86.0(3.5) & 81.7\\
                DAMA\_sup * & 78.7(2.4) & 87.3(2.2) & 91.5(2.6) & 88.6(4.3) & 77.4(3.2) & 85.9(2.4) & 90.7(3.3) & 88.2(4.1) & 79.6(2.2) & 88.8(1.6) & 90.1(3.6) & 89.4(4.1) & 78.5(2.6) & 87.4(2.0) & 89.1(3.1) & 88.6(4.7) & 86.2\\
                CDSPP\_sup (Ours) * & 80.3(2.0) & 89.0(1.9) & 92.0(3.5) & 90.7(3.8) & 80.3(2.1) & 89.1(1.9) & 91.7(3.7) & 90.7(3.7) & 79.8(2.1) & 88.9(1.8) & 90.4(3.9) & 90.1(3.9) & 80.4(2.2) & 89.0(1.8) & 91.5(4.1) & 90.6(3.8) & 87.8\\
                \hline
                DAMA \citep{wang2011heterogeneous}& 76.6(2.6) & 86.2(1.9) & 91.0(2.5) & 88.2(4.3) & 73.6(4.7) & 83.3(2.6) & 88.8(3.7) & 86.5(4.4) & 77.5(2.5) & 88.4(1.6) & 90.7(4.2) & 90.1(3.8) & 76.1(2.9) & 86.0(2.3) & 87.7(4.7) & 86.8(5.8) & 84.8\\
                SHFA \citep{li2013learning}& 77.1(2.8) & 86.2(3.8) & 93.0(3.6) & 90.0(2.6) & 80.5(3.1) & 86.7(2.2) & 94.3(2.5) & 90.0(4.0) & 81.6(2.1) & 88.5(2.9) & 93.5(3.9) & 92.0(4.1) & 80.5(1.8) & 88.5(2.4) & 93.5(3.5) & 89.5(4.2) & 87.8\\
                CDLS \citep{hubert2016learning} & 80.6(1.8) & 88.8(2.1) & 93.0(3.2) & 91.1(3.7) & 80.6(1.8) & 88.8(2.1) & 92.0(3.0) & 91.0(4.5) & 78.4(2.7) & 87.2(2.3) & 93.0(3.7) & 88.9(5.6) & 81.0(2.0) & 88.6(2.2) & 92.1(3.3) & 91.4(4.2) & 87.9\\
                PA \citep{li2018heterogeneous} & 87.2(1.1) & 90.8(1.3) & 92.9(3.3) & 93.9(3.9) & 87.0(1.1) & 90.5(1.7) & 94.7(2.5) & 94.0(3.9) & 87.0(1.3) & 90.5(2.0) & \bf 94.5(2.8) & 94.3(3.7) & 87.0(1.3) & 90.7(1.5) & 93.4(4.1) & 92.8(4.6) & 91.3\\
                TIT \citep{li2018transfer} & 84.9(1.7) & 89.9(1.6) & 94.6(3.1) & 92.2(4.3) & 84.6(1.5) & 89.7(1.7) & 94.6(2.2) & 92.3(4.9) & 82.7(1.5) & 88.7(1.9) & 94.3(2.7) & 92.1(4.0) & 84.7(1.6) & 89.5(1.8) & 92.5(2.8) & 92.5(4.3) & 90.0\\
                STN \citep{yao2019heterogeneous} & 88.2(1.7) & 92.4(0.7) & 94.4(2.0) & 92.8(4.9) & \bf 88.4(1.6) & 92.5(0.7) & 95.0(2.0) & 93.9(4.1) & 87.9(1.7) & 92.2(0.5) & 94.4(2.5) & 93.3(5.0) & \bf 88.2(1.8) & 92.6(0.8) & 93.9(3.2) & 92.2(5.1) & 92.0 \\
                DDACL \cite{yao2020discriminative} & 86.5(1.6) & 91.8(0.9) & 94.2(2.8) & 93.5(3.4) & 86.2(1.9) & 83.1(11.2) & 89.1(5.9) & 92.3(3.9) & 86.2(1.7) & 91.8(1.1) & 93.4(3.6) & 93.6(3.0) & 86.8(1.7) & 92.0(0.8) & 94.4(3.2) & 94.0(3.1) & 90.6\\
                SSAN \cite{li2020simultaneous} & 80.9(8.7)& 89.8(2.8)& \bf 95.8(2.0)& \bf 94.2(2.1)& 84.9(4.7)& 89.0(4.0)& 93.1(3.6)& 93.1(3.1)& 81.0(4.7)& 90.3(1.5)& 93.9(3.6)& 82.6(14.7)& 84.3(2.2)& 86.9(10.0)& 93.5(5.2)& \bf 95.0(2.1)& 89.3 \\
                DAMA+ & 88.1(1.7) & \bf 92.7(0.6) & 93.9(1.7) & 92.2(4.1) & 88.0(1.3) & \bf 92.9(0.6) & 93.9(2.1) & 92.8(4.2) & 87.7(1.9) & \bf 93.2(0.5) & 92.1(5.3) & 94.0(3.3) & 88.1(2.1) & \bf 92.7(0.7) & 94.8(1.6) & 93.5(3.9) & 91.9\\
                CDSPP (Ours) & \bf 88.3(0.7) & 92.3(0.7) & 95.6(1.5) & 94.1(4.1) & 88.1(1.0) & 92.6(0.5) & \bf 95.7(1.0) & \bf 94.6(3.8) & \bf 88.1(0.6) & 92.7(0.5) & 93.5(4.6) & \bf 94.7(3.5) & 88.1(1.0) & 92.5(0.5) & \bf 95.7(1.3) & 94.3(3.8) & \bf 92.6\\

                \hline
            \end{tabular}%
        }
    }
\end{table*}

\begin{table*}[!htbp]
    \centering
    {
        \centering
        \caption[]{Mean(std) of classification accuracy (\%) over ten trials on the Office-Home dataset using VGG16 (source) and ResNet50 (target) features under supervised (denoted by $*$) and semi-supervised settings (each column represents one Source $\to$ Target adaptation task).}
        \label{table:vgg2resnet}
        \resizebox{\columnwidth}{!}{%
            \begin{tabular}{l cccc cccc ccccc cccc}
                \hline
                Method & A$\to$A & A$\to$C & A$\to$P & A$\to$R & C$\to$A & C$\to$C& C$\to$P & C$\to$R & P$\to$A & P$\to$C & P$\to$P & P$\to$R & R$\to$A & R$\to$C & R$\to$P & R$\to$R & Avg \\ \hline
                SVM$_t *$& 51.8(1.2) & 41.4(1.6) & 71.0(1.4) & 65.8(2.3)& 51.8(1.2) & 41.4(1.6) & 71.0(1.4) & 65.8(2.3)& 51.8(1.2) & 41.4(1.6) & 71.0(1.4) & 65.8(2.3)& 51.8(1.2) & 41.4(1.6) & 71.0(1.4) & 65.8(2.3) & 57.5\\
                CDLS\_sup \citep{hubert2016learning} $*$ & 58.7(0.9) & 45.7(1.5) & 75.0(0.8) & 69.8(1.9) & 53.4(1.0) & 48.6(1.0) & 73.9(0.9) & 67.8(1.8) & 55.0(0.9) & 45.9(1.4) & 78.0(0.8) & 70.2(1.5) & 56.5(1.1) & 46.8(1.5) & 76.2(0.5) & 72.4(1.4) & 62.1\\
                DAMA\_sup * & 56.6(2.8) & 43.6(2.2) & 72.0(1.4) & 67.8(2.4) & 42.7(4.8) & 39.8(5.4) & 64.8(5.9) & 57.5(4.5) & 52.4(3.9) & 40.4(4.1) & 70.1(5.7) & 63.6(3.8) & 51.8(3.6) & 42.4(3.4) & 68.8(5.1) & 65.5(4.7) & 56.2\\
                CDSPP\_sup (Ours)  $*$ & 60.8(1.2) & 49.5(1.1) & 76.3(0.8) & 71.9(1.8) & 59.4(1.4) & 50.4(1.0) & 76.1(0.9) & 71.6(1.8) & 59.8(1.2) & 49.6(1.1) & 78.0(0.9) & 72.4(1.4) & 60.4(1.3) & 49.8(0.9) & 76.9(1.0) & 73.3(1.6) & 64.8\\
                \hline
                DAMA \citep{wang2011heterogeneous} & 55.6(3.3) & 43.8(2.1) & 71.1(2.1) & 66.4(3.5) & 43.1(4.7) & 39.3(5.2) & 62.9(5.7) & 56.4(4.7) & 52.1(4.1) & 40.4(4.6) & 69.9(4.3) & 64.3(5.3) & 51.9(3.6) & 42.0(4.3) & 68.3(5.0) & 65.1(4.5) & 55.8\\
                CDLS \citep{hubert2016learning} & 62.1(0.9) & 46.9(1.2) & 76.8(0.7) & 71.5(2.3) & 55.7(1.3) & 47.4(1.2) & 76.7(0.6) & 70.8(2.0) & 56.4(1.1) & 47.0(1.2) & 77.8(0.6) & 71.5(2.0) & 56.7(1.2) & 47.6(1.3) & 77.5(0.4) & 72.2(2.0) & 63.4\\
                PA \citep{li2018heterogeneous} & 59.8(1.2) & 48.2(1.5) & 80.0(1.2) & 75.5(1.8) & 59.8(1.1) & 48.2(1.3) & 80.0(1.3) & 75.4(1.9) & 59.5(1.5) & 48.2(1.4) & 80.0(1.6) & 75.7(1.9) & 59.6(1.3) & 48.2(1.5) & 79.9(1.4) & 75.7(1.8) & 65.8\\
                TIT \citep{li2018transfer} & 55.6(1.0) & 44.7(1.3) & 74.3(1.0) & 70.3(1.8) & 56.1(0.9) & 45.5(1.1) & 74.7(0.7) & 70.2(1.7) & 55.9(1.1) & 45.3(1.3) & 74.9(0.9) & 70.2(1.8) & 55.5(1.5) & 44.6(1.4) & 74.7(0.8) & 69.9(2.0) & 61.4\\
                STN \citep{yao2019heterogeneous} & 62.6(1.4) & 51.2(1.5) & 78.7(3.9) & 74.5(4.3) & 56.1(3.8) & 52.2(2.2) & 77.0(4.0) & 71.1(6.0) & 60.7(1.3) & 49.3(6.0) & \bf 82.4(1.0) & 75.8(2.8) & 61.0(1.3) & 50.6(3.2) & 80.4(0.9) & 75.7(4.4) & 66.2 \\
                DDACL \cite{yao2020discriminative} & 50.3(2.2) & 39.8(2.4) & 59.4(2.8) & 56.1(3.4) & 45.1(2.0) & 36.3(3.0) & 60.9(2.9) & 56.8(2.0) & 40.3(1.5) & 34.2(2.3) & 55.7(9.1) & 43.0(9.9) & 41.9(2.4) & 36.5(2.0) & 52.4(5.1) & 51.5(9.2) & 47.5\\
                SSAN \cite{li2020simultaneous}& 50.5(1.9)& 40.1(3.0)& 70.9(1.8)& 63.9(3.0)& 43.9(2.9)& 42.5(5.0)& 67.8(1.2)& 61.9(2.9)& 44.1(2.6)& 38.1(3.5)& 77.3(0.9)& 66.2(1.3)& 45.7(3.9)& 38.6(3.8)& 71.7(4.0)& 68.8(2.5)& 55.8 \\
                DAMA+ & 62.1(2.4) & 49.0(1.4) & 77.7(1.9) & 75.0(2.5) & 54.0(5.2) & 44.7(6.1) & 75.6(3.7) & 69.0(3.4) & 60.9(2.7) & 46.9(3.1) & 76.9(3.5) & 72.5(1.9) & 60.3(1.9) & 48.6(3.7) & 76.7(2.8) & 73.4(3.3) & 63.9\\
                CDSPP (Ours) & \bf 65.7(1.0) & \bf 54.8(2.0) & \bf 81.0(1.5) & \bf 78.4(1.1) & \bf 65.0(1.4) & \bf 55.1(1.6) & \bf 80.9(1.6) & \bf 78.5(1.2) & \bf 65.6(0.4) & \bf 54.7(1.9) & 81.5(1.1) & \bf 78.8(1.0) & \bf 65.5(0.9) & \bf 54.6(1.6) & \bf 80.9(1.6) & \bf 79.4(0.9) & \bf 70.0\\

                \hline
            \end{tabular}%
        }
    }
\end{table*}

\begin{table*}[!htbp]
    \centering
    {
        \centering
        \caption[]{Mean(std) of classification accuracy (\%) over ten trials on the Office-Home dataset using ResNet50 (source) and VGG16 (target) features under supervised (denoted by $*$) and semi-supervised settings (each column represents one Source $\to$ Target adaptation task).}
        \label{table:resnet2vgg}
        \resizebox{\columnwidth}{!}{%
            \begin{tabular}{l cccc cccc ccccc cccc}
                \hline
                Method & A$\to$A & A$\to$C & A$\to$P & A$\to$R & C$\to$A & C$\to$C& C$\to$P & C$\to$R & P$\to$A & P$\to$C & P$\to$P & P$\to$R & R$\to$A & R$\to$C & R$\to$P & R$\to$R & Avg \\ \hline
                SVM$_t$ * & 40.3(1.4) & 30.5(1.6) & 63.3(1.7) & 56.3(2.9) & 40.3(1.4) & 30.5(1.6) & 63.3(1.7) & 56.3(2.9) & 40.3(1.4) & 30.5(1.6) & 63.3(1.7) & 56.3(2.9) & 40.3(1.4) & 30.5(1.6) & 63.3(1.7) & 56.3(2.9) & 47.6\\
                CDLS\_sup \citep{hubert2016learning} * & 51.4(1.1) & 36.5(1.0) & 69.6(1.1) & 63.5(2.0) & 46.4(1.2) & 39.2(1.0) & 68.7(1.2) & 62.0(1.6) & 47.2(1.2) & 36.4(0.8) & 73.1(1.0) & 64.6(1.9) & 48.6(1.1) & 37.1(1.1) & 70.9(1.2) & 66.4(2.0) & 55.1\\
                DAMA\_sup * & 46.9(1.8) & 35.6(1.8) & 65.9(1.4) & 60.3(1.8) & 43.4(2.4) & 32.5(3.7) & 60.3(6.0) & 56.3(3.0) & 44.1(4.0) & 31.8(3.6) & 62.2(4.0) & 56.4(4.0) & 45.3(3.2) & 34.4(1.6) & 60.9(4.6) & 60.3(2.1) & 49.8\\

                CDSPP (Ours)* & 49.7(1.1) & 39.2(1.0) & 69.5(1.3) & 63.7(2.0) & 48.3(1.2) & 40.4(1.3) & 69.5(1.5) & 63.4(1.8) & 48.5(1.1) & 38.9(0.8) & 71.3(1.4) & 64.1(1.9) & 49.0(1.2) & 39.4(1.1) & 70.1(1.3) & 65.0(2.1) & 55.6\\
                \hline
                DAMA \citep{wang2011heterogeneous} & 46.7(2.0) & 33.6(2.5) & 66.2(1.7) & 57.8(3.4) & 43.1(4.0) & 32.0(4.5) & 60.2(6.2) & 55.7(5.0) & 44.3(3.7) & 32.0(4.1) & 65.5(5.6) & 59.8(3.5) & 45.3(3.4) & 34.8(2.6) & 65.0(4.4) & 60.9(3.5) & 50.2\\

                CDLS \citep{hubert2016learning} & 54.9(1.1) & 36.6(1.1) & 71.1(0.8) & 65.9(1.3) & 47.8(1.4) & 39.8(1.2) & 69.5(1.2) & 63.6(1.4) & 49.7(1.2) & 36.8(1.2) & 75.6(0.8) & 67.9(1.6) & 52.3(1.0) & 38.5(1.3) & 73.1(1.0) & 69.6(1.6) & 57.0\\
                PA \citep{li2018heterogeneous} & 51.4(1.0) & 38.3(1.3) & 73.7(1.2) & 67.4(1.6) & 51.2(1.4) & 38.2(1.2) & 73.6(1.2) & 67.4(1.6) & 51.2(1.1) & 38.1(1.4) & 73.6(1.2) & 67.3(1.9) & 51.2(0.9) & 38.2(1.2) & 73.7(1.2) & 67.4(1.4) & 57.6\\
                TIT \citep{li2018transfer} & 46.8(1.7) & 36.4(1.2) & 69.4(0.9) & 62.5(1.8) & 47.0(1.7) & 36.4(1.1) & 69.3(1.1) & 62.0(2.2) & 46.8(1.7) & 36.4(1.1) & 69.8(0.9) & 62.4(2.1) & 45.9(1.6) & 36.0(1.3) & 69.4(1.2) & 62.5(2.1) & 53.7\\
                STN \citep{yao2019heterogeneous} & 52.6(1.5) & 41.2(2.4) & 74.9(1.0) & 69.2(1.5) & 51.2(1.1) & 42.5(1.2) & 75.3(1.2) & 69.6(1.0) & 53.0(1.2) & 41.7(1.4) & \bf 77.3(1.2) & 70.7(1.4) & 52.7(1.9) & 41.7(1.4) & \bf 76.6(1.0) & 71.6(1.3) & 60.1 \\
                DDACL \cite{yao2020discriminative} & 33.8(2.3) & 27.5(1.6) & 52.2(4.0) & 46.8(1.6) & 31.8(2.3) & 24.3(1.6) & 50.8(1.9) & 44.0(3.5) & 32.0(2.7) & 23.4(2.8) & 49.0(7.9) & 39.9(7.3) & 32.4(2.7) & 24.9(1.6) & 46.5(3.7) & 45.5(4.4) & 37.8\\
                SSAN\cite{li2020simultaneous}& 42.2(4.1)& 30.4(2.3)& 61.9(3.7)& 56.5(2.6)& 37.9(1.6)& 32.3(2.3)& 62.1(1.5)& 53.4(3.4)& 38.1(2.0)& 29.9(1.7)& 69.0(2.9)& 58.0(1.9)& 37.5(2.3)& 29.6(1.7)& 63.3(2.2)& 57.9(3.5)& 47.5 \\
                DAMA+ & 49.1(2.9) & 37.5(1.2) & 71.1(1.5) & 65.4(2.5) & 49.7(1.9) & 32.9(4.1) & 68.3(3.3) & 63.2(3.7) & 48.9(3.1) & 33.3(3.6) & 68.1(2.5) & 61.4(3.9) & 49.9(3.1) & 36.3(2.2) & 67.1(2.4) & 64.6(2.1) & 54.2\\
                CDSPP (Ours) & \bf 55.6(1.1) & \bf 44.7(1.8) & \bf 75.2(1.6) & \bf 71.7(1.4) & \bf 54.5(1.2) & \bf 46.0(1.6) & \bf 75.7(1.6) & \bf 71.4(1.9) & \bf 54.7(1.2) & \bf 45.0(1.6) & 76.0(1.8) & \bf 71.8(1.6) & \bf 55.0(1.3) & \bf 44.9(2.0) & 75.8(1.8) & \bf 72.1(1.8) & \bf 61.9\\

                \hline
            \end{tabular}%
        }
    }
\end{table*}

Table \ref{table:mrc-tag2image} (except the last column) lists the comparison results on the MRC dataset. The baseline method SVM$_t$ achieves an accuracy of 67.0\% using only 10 labelled target domain samples per class for training. The labelled source domain data can benefit the performance with proper domain adaptation but the improvement is marginal for both HFA and our proposed CDSPP. The supervised version of CDLS uses PCA to learn a subspace from the target domain, hence the dimensionality of subspace cannot be higher than $n_t-1$. Due to such limitation, CDLS\_sup performs worse than others when the number of labelled target samples is small which is usually the case for HDA problems.  For the semi-supervised HDA, DAMA and SHFA perform no better than the baseline method SVM$_t$ which was also observed in existing works \citep{hubert2016learning,li2018heterogeneous,li2018transfer}. The best performance (71.7\%) is achieved by PA \citep{li2018heterogeneous} and our proposed CDSPP is marginally worse with the average classification accuracy of 68.9\%. However, when applying PCA to reduce the text features to a lower dimensionality of 50, the performance of CDSPP is improved from 68.9\% to 71.6\%, comparable with the best performance 71.7\% achieved by PA. This demonstrates the fact handcrafted text features (i.e. bag-of-features) used in the MRC dataset contain noisy variables which cannot be well handled by the CDSPP algorithm itself but a pre-processing like PCA suffices to address this issue.

Table \ref{table:mrc-tag2image} (rightmost column) also presents the results of tag-to-image adaptation on the NUS-ImageNet dataset. There is only one adaptation task (i.e. Tag$\to$Image) in this dataset. In the supervised HDA setting, the baseline method SVM$_t$ is outperformed by all three comparative methods with large margins among which HFA achieves the best performance of 67.5\% as opposed to the accuracy of 67.2\%  by our proposed CDSPP\_sup. However, HFA is more computationally expensive than others as discussed in \citep{li2013learning}. In the semi-supervised HDA setting, our method achieves the best performance with an accuracy of 74.7\%. The performance of our CDSPP can be further improved to 76.5\% when PCA is applied to reduce the dimensionality of the text features to 50.

Similar results can also be observed in Table \ref{table:surf2decaf} for the image classification experiments on Office-Caltech. Both HFA and our CDSPP achieve the same average accuracy of 87.8\% in the supervised HDA setting. CDLS\_sup performs worse than the baseline method SVM\_t again due to the restricted PCA dimensions as discussed above. In the semi-supervised HDA, our CDSPP achieves the best results in 6 out of 16 adaptation tasks and has the highest average accuracy of 92.6\%.

The experimental results for the challenging Office-Home dataset are shown in Table \ref{table:vgg2resnet} and Table \ref{table:resnet2vgg}. The difference between these two tables lies in the features used for the source/target domains are VGG16/ResNet50 and ResNet50/VGG16 respectively. In this experiment, the methods HFA and SHFA are excluded due to their extremely long computation time given the scale of this dataset. It can be seen that CDLS\_sup, for the first time, outperforms the baseline method SVM$_t$ on this dataset since the total number of labelled target samples is 195 which no longer restricts the PCA dimension in this algorithm. Two more recent approaches DDACL \cite{yao2020discriminative} and SSAN \cite{li2020simultaneous}, however, perform poorly on this more challenging dataset although they achieve good performance on three simpler datasets. One reasonable explanation is that these two approaches along with many others benefit from the clustering characteristics of the original features and can easily recognize the target samples cluster-wisely. For the more challenging dataset, the classes are prone to overlap in a low-dimensional subspace if the projections are not properly learned. The simultaneous learning of the classifier and feature projections tends to result in an overfitted classifier to the labelled and pseudo-labelled samples and the overfitting can be an issue when the labelled target samples cannot represent the distribution of their corresponding classes in the subspace. As a result, they suffer from negative adaptation when the pseudo-labels are inaccurate at the beginning and less robustness to the choice of labelled target samples. This also provides evidence for the necessity of new test beds for HDA approaches.  In both tables, the best performances were achieved by our CDSPP for most adaptation tasks in both supervised and semi-supervised settings. Specifically, CDSPP achieves an average accuracy of 70.0\% when VGG16 and ResNet50 features were employed for source and target domains, significantly better than the second-best performance 66.2\% achieved by STN \citep{yao2019heterogeneous}. Similar results can be observed in Table \ref{table:resnet2vgg}, CDSPP achieves the best performance of 61.9\% as opposed to the second-best 60.1\% by STN \citep{yao2019heterogeneous}. The significant performance improvement gained by CDSPP on the Office-Home dataset is attributed to the fact this dataset is much more challenging than other datasets since it contains much more classes (65 vs 10, 8, 6). We believe Office-Home is a more appropriate testbed for discriminating different HDA methods.

In addition, the performance comparison between DAMA and DAMA+ provide further evidence that the use of the iterative learning framework described in Section \ref{sec:shda} is beneficial to semi-supervised HDA. On the other hand, the superior performance of CDSPP to DAMA+ across all datasets validates the fact that our CDSPP is essentially different from DAMA as discussed in Section \ref{sec:recognition}. In the supervised HDA experiments, CDSPP also outperforms our adaptation of DAMA consistently on four datasets and the performance gap on the challenging Office-Home dataset is particularly significant. The other interesting phenomenon that can be observed from Tables \ref{table:surf2decaf}-\ref{table:vgg2resnet} is the semi-supervised DAMA (i.e. the original version in \cite{wang2011heterogeneous}) performs no better than its supervised version (i.e. DAMA\_sup adapted by ourselves). This demonstrates that the way how DAMA \cite{wang2011heterogeneous} exploits the unlabelled target-domain data is ineffective. By contrast, the selective pseudo-labelling strategy employed in our proposed CDSPP is more effective and can be readily used by other HDA algorithms.

\subsection{On the Number of labelled Target Samples}
We conducted additional experiments of semi-supervised HDA to compare our proposed CDSLPP with other methods when different numbers of labelled target samples were used for training. Specifically, we set the number of labelled target samples as 5, 10, 15 or 20 for the MRC dataset whilst for the other three datasets the investigated numbers of labelled target samples were within the collection of $\{1, 3, 5, 7, 9\}$. For the MRC and NUS-ImageNet datasets, all adaptation tasks (i.e. $EN/FR/GE/IT \to SP$ and $Tag\to Image$, respectively) were repeated for ten trials with randomly selected data (the same as those used in the previous experiment). To save computational time without loss of generality, we only conducted the first four adaptation tasks for the first three trials for the Office-Caltech ($C\to C, C\to A, C\to D, C\to W$) and Office-Home ($A\to A, A\to C, A\to P, A\to R$ 
with VGG16 and ResNet50 as the source and target features, respectively) datasets in this experiment. For each dataset, the average classification accuracy over all the conducted adaptation tasks in this dataset is reported for comparison.

\begin{figure*}
    \centering
    {\includegraphics[width=1\textwidth]{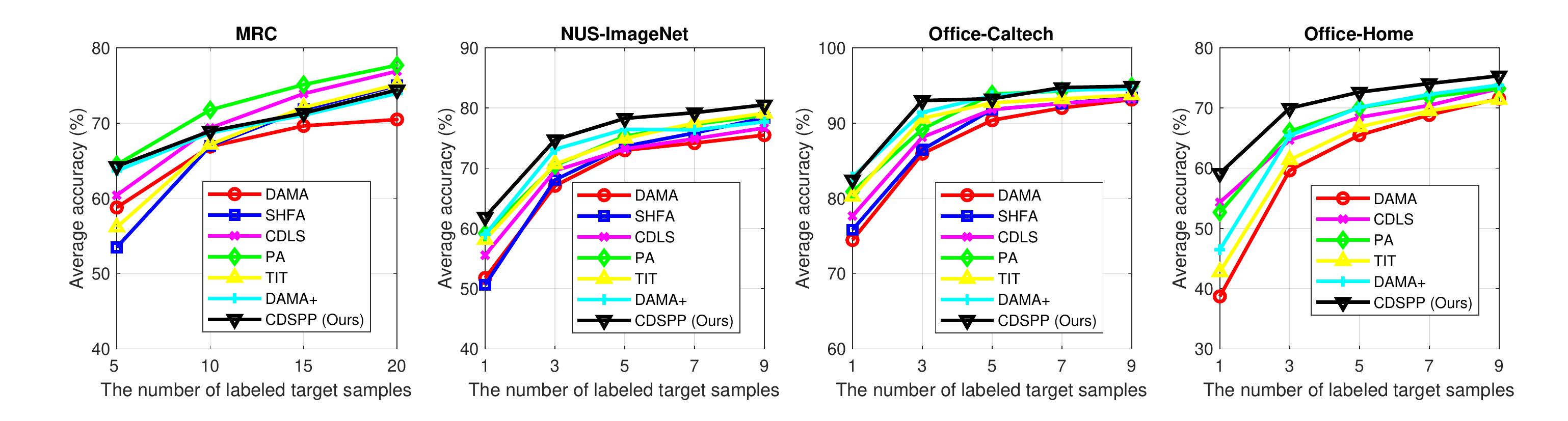}}
    {\caption{Comparison results when different numbers of labelled target samples are used. }
        \label{fig:numOfLTS}}
\end{figure*}

The experimental results are shown in Figure \ref{fig:numOfLTS} from which we can draw some conclusions. (1) The performance of all methods is improved with the increase of labelled target samples since more labelled target samples provide additional information for the training. (2) The performance margins between different methods decrease when more labelled target samples are used for training. This phenomenon demonstrates these methods have different capabilities of cross-domain knowledge transfer which is of vital importance when there are limited labelled data in the target domain. (3) Our proposed CDSPP algorithm outperforms the others in three out of four datasets regardless of the number of labelled target samples. The superiority of CDSPP to other methods is more significant when less labelled target samples are available. (4) On the MRC dataset, our method performs the best when 5 labelled target samples are used but outperformed by CDLS \citep{hubert2016learning} and \citep{li2018heterogeneous} when more labelled target samples are available.

\subsection{On the Effect of Hyper-parameters}

In all our experiments described above, we empirically set the dimensionality of the common subspace $d$ equal to the number of classes in the dataset and set the hyper-parameters $\alpha=10$ (c.f. Eq.(\ref{eq:eig})) and the number of iterations $T=5$ (c.f. Algorithm \ref{alg:hda}). In this experiment, we will show how these values were selected and the fact that our algorithm is not sensitive to these hyper-parameters across all the datasets. Similar to the experimental settings in the previous section, we repeated all the adaptation tasks for ten trials for the MRC and NUS-ImageNet datasets and repeated the first four adaptation tasks for the first three trials for the Office-Caltech and Office-Home datasets to save time without loss of generality. The average accuracy over all the investigated adaptation tasks is reported for each dataset when a specific hyper-parameter value is used. 

Firstly, we investigate the effect of the subspace dimension $d$. The values of $d$ were from the set $\{ 2, 4, 6, 8, 10, 16, 32, 64/65, 128, 256, 512 \}$ which contains the class numbers of four datasets (i.e. 6, 8, 10 and 65) as well as other candidate values less or greater than the class numbers. The experimental results are shown in the left graph of Figure \ref{fig:sensitivity}. It is not hard to see that the best performance can be achieved when the value of $d$ is no less than the number of classes in each dataset. A greater value of $d$ does not further improve the performance but a smaller value of $d$ leads to a significant performance drop. 
As a result, it is easy to select an optimal value of the subspace dimension for our proposed CDSPP.

Subsequently, We investigate the effect of the regularization parameter $\alpha$ in Eq.(\ref{eq:eig}) by conducting experiments with the values of $\alpha$ selected from $\{0.01, 0.1, 1, 10, 100, 1000\}$. The experimental results are shown in the middle graph of Figure \ref{fig:sensitivity} from which we can see that the optimal values of $\alpha$ should be between 10 and 100 across all datasets. A smaller value of $\alpha$ leads to performance drops for all datasets except Office-Caltech. This validates the necessity of the regularization term in Eq.(\ref{eq:eig}) in our method and it is not very sensitive to the value of $\alpha$. Similar findings have been validated in the traditional LPP algorithm by \citet{wang2017zero}.

Finally, we are concerned about the number of iterations $T$ by setting $T=\{1,3,5,7,9,11,15,21\}$. The right-side graph in Figure \ref{fig:sensitivity} shows that the CDSPP algorithm performs generally well when $T\geq 5$. Increasing the number of iterations further can only improve the performance on the NUS-ImageNet dataset very marginally but will increase the computational cost significantly. As a result, we selected $T=5$ as the optimal value in all our experiments.

\begin{figure*}
    \centering
    {\includegraphics[width=1\textwidth]{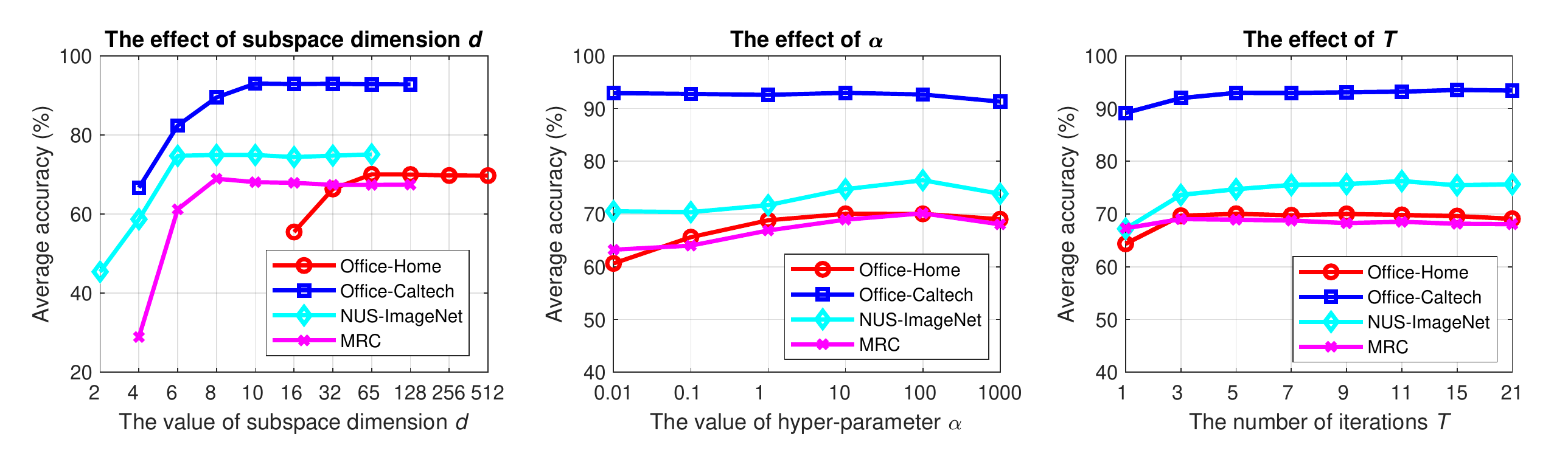}}
    {\caption{Performance sensitivity to hyper-parameters. }
        \label{fig:sensitivity}}
\end{figure*}

\subsection{Qualitative Evaluation} \label{sec:qualitative}
To give an intuitive explanation of how our algorithm can align two heterogeneous domains progressively, we take the tag-to-image adaptation task in the NUS-ImageNet dataset as an example and visualise the distribution of samples in the learned subspace. As shown in Figure \ref{fig:visualisation}(a), the original features from the two domains are independent of each other although the clustering characteristics are evident. Figure \ref{fig:visualisation}(b) illustrates how the three labelled target samples (``circles") are pulled closer to the corresponding source classes (``squares") after the first iteration of CDSPP. More importantly, due to the property of structure preservation of CDSPP, the unlabelled target samples (``crosses") are also moving towards their corresponding source clusters. In Figure \ref{fig:visualisation}(c), we can see more target samples are pseudo-labelled (``crosses" within ``circles") and the source and target domains are further aligned. Such progressive pseudo-labelling and domain alignment are enhanced in Figure \ref{fig:visualisation}(d) and no significant improvement can be observed in the following iterations (e) and (f). This is consistent with the recognition results achieved by our CDSPP in this particular experiment (i.e. from the first to the fifth iteration, recognition accuracy is 70.1\%, 76.7\%, 79.1\%, 78.9\% and 79.0\%, respectively).

It is obvious that the clustering of eight classes has converged after the third iteration and the two domains are relatively well aligned. The samples which are misclassified in the final iteration are those located in the overlapping regions of two classes. The overlap comes from the original features as shown in Figure \ref{fig:visualisation}(a) and can be mitigated in different ways. The best way is to extract more discriminative features to avoid such distribution overlap from the beginning which, however, is beyond our focus of this paper. Alternatively, one can use a more capable domain adaptation algorithm such as our proposed CDSPP to mitigate the class overlap by learning the most discriminative features from the original ones. In addition, the choice of labelled target samples also makes a difference. Taking a closer look at Figure \ref{fig:visualisation}(a), we can see one of the three randomly selected labelled target samples for class 5 is far away from the target cluster of class 5. When this outlier is pulled closer to the source cluster of class 5, some samples from class 2 and class 6 are also mistakenly pulled close to the source cluster of class 5 as shown in Figure \ref{fig:visualisation}(b). These observations also imply it is important to choose the most representative target samples to label for improved performance in practice.

\begin{figure*}[ht!]
    \centering
    \includegraphics[width=\textwidth]{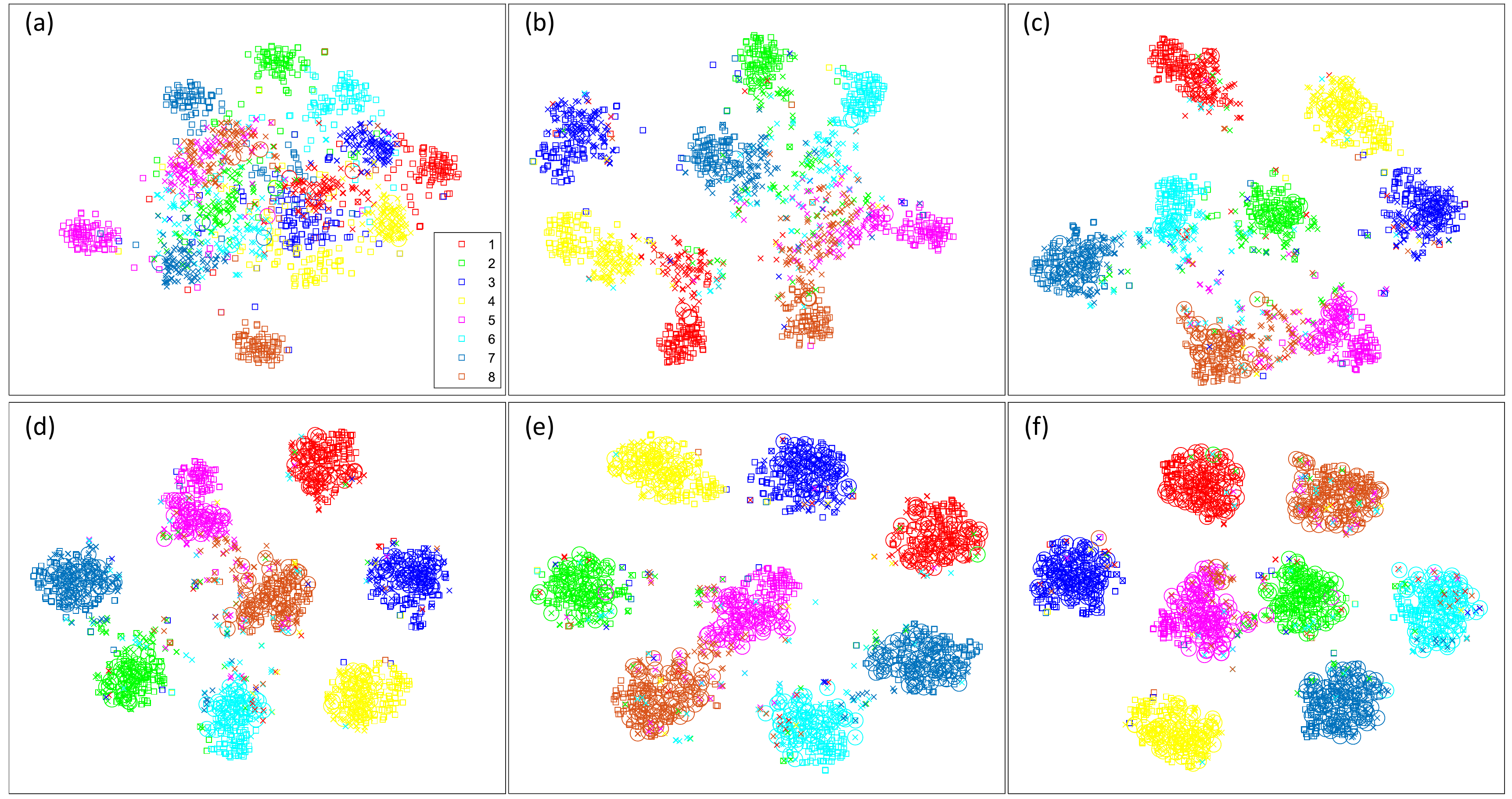}
    \caption{Visualisation of the learned subspace for the NUS-ImageNet dataset (i.e. the tag to image adaptation task) using the proposed CDSPP, best view in colour. (Results are from one of the ten trials with a specific random seed; eight classes 1-8 are represented by different colours; ``squares": labelled source samples; ``crosses": unlabelled target samples; ``circles": labelled or pseudo-labelled target samples; (a) the original features learned by two separate PCA projections independently; (b)-(f) projections in the subspace learned by CDSPP after 1st-5th iteration.)}
    \label{fig:visualisation}
\end{figure*}

\subsection{On the Computational Efficiency}\label{sec:comptime}
We compare the computational efficiency of different methods by calculating the time cost of each method in the experiments. The experiments are conducted on a laptop with an Intel Core i5-7300HQ CPU @ 2.5 GHz and 32 GB memory. For neural network based methods STN and SSAN, the Nvidia Titan Xp GPUs are used. The results are shown in Table \ref{table:time}. The computational time is calculated by averaging the time for all adaptation tasks (i.e. 4, 1, 16 and 16 tasks for MRC, NUS-ImageNet, Office-Caltech and Office-Home respectively) over three trials. By comparison, our proposed CDSPP is generally the most efficient method on three out of four datasets. The exception on Office-Caltech is because CDLS and TIT use dimensionality reduction such as PCA to reduce the dimensionality of Decaf features from 4096 to a much lower value whilst our CDSPP uses the original 4096-dimensional features. From Table \ref{table:time} we can also see different methods have the varying capability of scaling to larger datasets (e.g., from NUS-ImageNet to Office-Home) in terms of both feature dimensionality and the number of samples. In particular, SHFA takes an excessively long time before completing one single adaptation task of Office-Home in our experiment hence is marked as $Inf$ in the table. STN and SSAN take the most time across all datasets since neural networks are trained for a large number of iterations which is generally much less efficient compared with our CDSPP which can be solved by eigen-decomposition.

\begin{table}[!t]
    \centering
    {
        \centering
        \caption[]{Computation time (s) of different methods on four datasets (the total time of all adaptation tasks in each dataset is calculated).
        }
        \label{table:time}
        \resizebox{0.9\columnwidth}{!}{%
            \begin{tabular}{lrrrr}
                \hline
                Method & MRC & NUS-ImageNet & Office-Caltech & Office-Home \\ \hline
                DAMA \citep{wang2011heterogeneous}  & 46 & 7 & 58 & 477\\
                SHFA \citep{li2013learning}& 917 & 25 & 255 & Inf \\
                CDLS \citep{hubert2016learning} & 168 & \bf 6 & \bf 47 & 272\\
                PA \citep{li2018heterogeneous}& 617 & 30 & 121 & 3991\\
                TIT \citep{li2018transfer} & 175 & 11 & 52 & 1740\\
                STN \citep{yao2019heterogeneous} & 2734 & 343 & 7134 & 40857 \\
                DDACL \citep{yao2020discriminative} & 622 & 169 & 2940 & 3421 \\
                SSAN \citep{li2020simultaneous} & 9520 & 1229 & 13245 & 47145 \\
                DAMA + & 49 & 21 & 288 & 1390\\
                CDSPP (Ours)  & \bf 16 & 7 & 161 & \bf 256\\
                \hline
            \end{tabular}%
        }
    }
\end{table}

\section{Conclusion and Future Work} \label{sec:conclusion}
We propose a novel algorithm CDSPP for HDA and extend it to the semi-supervised setting by incorporating it into an iterative learning framework. Experimental results on several benchmark datasets demonstrate the proposed CDSPP is not only computationally efficient but also can achieve state-of-the-art performance on four datasets. We also investigate the effect of the number of labelled target samples in the performance of different methods and found that the use of too many labelled target samples will suppress the performance distinction among different methods. The newly introduced benchmark dataset Office-Home for HDA is proved a proper testbed for HDA since it is more challenging with much more classes than others and the performances of investigated methods on this dataset are more significantly varied. In addition, the proposed method for HDA is not sensitive to hyper-parameters and it is easy to select optimal hyper-parameter values across varying datasets. 

One limitation of the proposed method is that its performance relies on the quality of pre-extracted features. As we have observed in our experiments on the MRC dataset, proper pre-processing of features can affect the domain adaptation performance significantly. One direction of future work to address this issue is to unify the feature extraction neural networks and domain adaptation.  For HDA, the source and target domains are different either in the data modality (e.g., text and image) or in the feature space. As a result, two individual neural networks are needed for feature extraction before feeding the features into the domain adaptation module. Our selective pseudo-labelling strategy described in this paper can also be easily applied to exploit the unlabelled target-domain data when training the unified neural networks for HDA.
\bibliographystyle{apa}	

\bibliography{mybibfile}  
\end{document}